\documentclass[runningheads]{llncs}
\usepackage{amssymb}
\usepackage{amsmath}
\newtheorem{mytheorem}{Theorem}
\newtheorem{myexercise}[mytheorem]{Exercise}
\newtheorem{mydefinition}[mytheorem]{Definition}
\newtheorem{mylemma}[mytheorem]{Lemma}
\newtheorem{mycorollary}[mytheorem]{Corollary}
\def\LH{\Upsilon}

\def\sgn{\mathrm{sgn}}

\title{Reward-Punishment Symmetric\\Universal Intelligence}
\titlerunning{Reward-Punishment Symmetric Universal Intelligence}

\author{Samuel Allen Alexander\inst{1}
\and
Marcus Hutter\inst{2}}
\authorrunning{S.\ A.\ Alexander \& M.\ Hutter}

\institute{The U.S.\ Securities and Exchange Commission
\email{samuelallenalexander@gmail.com}
\url{https://philpeople.org/profiles/samuel-alexander/publications}
\and
DeepMind \& AMU
\url{http://www.hutter1.net/}
}

\begin{document}

\maketitle

\begin{abstract}
    Can an agent's intelligence level be negative?
    We extend the Legg-Hutter agent-environment framework to include punishments
    and argue for an affirmative answer to that question.
    We show that if the background encodings and Universal Turing Machine (UTM) admit
    certain Kolmogorov complexity symmetries,
    then the resulting Legg-Hutter intelligence measure is symmetric about
    the origin. In particular, this implies reward-ignoring agents
    have Legg-Hutter intelligence $0$ according to such UTMs.
\end{abstract}

\section{Introduction}

In their paper \cite{legg2007universal}, Legg and Hutter write:
\begin{quote}
    ``As our goal is to produce a definition of intelligence that is as broad and
    encompassing as possible, the space of environments used in our definition should
    be as large as possible.''
\end{quote}
So motivated, we investigate what would happen if we extended the universe
of environments to include environments with rewards from $\mathbb Q\cap [-1,1]$
instead of just from $\mathbb Q\cap [0,1]$ as in Legg and Hutter's paper.
In other words, we investigate what would happen if environments are not only
allowed to reward agents but also to punish agents (a punishment being a negative
reward).

We discovered that when negative rewards are allowed, this
introduces a certain algebraic structure into the agent-environment framework. The
main objection we anticipate to our extended framework
is that it implies the negative intelligence of certain
agents\footnote{Thus, this paper falls under the broader
program of advocating for intelligence measures having different ranges than
the nonnegative reals. Alexander has advocated
more extreme extensions of the range of intelligence measures
\cite{alexander2020archimedean} \cite{alexander2021measuring}; by contrast,
here we merely question the
assumption that intelligence never be negative, leaving aside the
question of whether intelligence should be real-valued.}.
We would argue that this makes perfect sense when environments are capable of punishing
agents: if the intelligence level of a reinforcement learning agent is a measure of
its ability to extract large rewards on average across many environments, then
an agent who instead extracts large punishments should have a negative intelligence level.

This paper advances the practical pursuit of AGI by suggesting
(in Section \ref{mainsecn})
certain symmetry constraints which would narrow down the space of
background UTMs, thereby refining at least one approach to
intelligence measurement. In particular these constraints are one
answer to Leike and Hutter, who asked: ``But what are other desirable
properties of a UTM?'' \cite{leike2015bad}.

The structure of the paper is as follows:
\begin{itemize}
    \item In Section \ref{prelimsecn}, we give preliminary definitions.
    \item In Section \ref{dualsection}, we introduce what we call the dual of an agent and of
        an environment, and prove some algebraic theorems about these.
    \item In Section \ref{mainsecn}, we show the existence of UTMs yielding Kolmogorov
        complexities with certain symmetries, and show that the resulting Legg-Hutter
        intelligence measures are symmetric too.
    \item In Section \ref{absvaluesection} we consider the absolute value of Legg-Hutter
        intelligence as an alternative intelligence measure.
    \item In Section \ref{conclusionsecn}, we summarize and make concluding remarks, including
        remarks about how these ideas might be applied to certain other intelligence measures.
\end{itemize}

\section{Preliminaries}
\label{prelimsecn}

In defining agent and environment below, we attempt to follow
Legg and Hutter \cite{legg2007universal} as closely as possible,
except that we permit environments to output rewards from $\mathbb Q \cap [-1,1]$
rather than just $\mathbb Q\cap [0,1]$ (and, accordingly, we modify which well-behaved
environments to restrict our attention to).

Throughout the paper, we implicitly
fix a finite set $\mathcal A$ of \emph{actions},
a finite set $\mathcal O$ of \emph{observations},
and a finite set $\mathcal R\subseteq \mathbb Q\cap [-1,1]$ of \emph{rewards}
(so each reward is a rational number between $-1$ and $1$ inclusive),
with $|\mathcal A|>0$,
$|\mathcal O|>0$, $|\mathcal R|>0$.
We assume that $\mathcal R$ has the following property:
whenever $\mathcal R$ contains any reward $r$, then $\mathcal R$
also contains $-r$.
We assume $\mathcal A$, $\mathcal O$, and $\mathcal R$ are mutually disjoint
(i.e., no reward is an action, no reward is an observation, and no action is an
observation).
By $\langle\rangle$ we mean the empty sequence.

\begin{mydefinition}
\label{omnibusdefn}
    (Agents, environments, etc.)
    \begin{enumerate}
        \item
        By $(\mathcal O\mathcal R\mathcal A)^*$ we mean the set of
        all finite sequences starting with an observation, ending with an action,
        and following the pattern ``observation, reward, action, ...''.
        We include $\langle\rangle$ in this set.
        \item
        By $(\mathcal O\mathcal R\mathcal A)^* \mathcal O\mathcal R$
        we mean the set of all sequences of the form $s\frown o\frown r$ where
        $s\in (\mathcal O\mathcal R\mathcal A)^*$, $o\in\mathcal O$
        and $r\in\mathcal R$ ($\frown$ denotes concatenation).
        \item
        By an \emph{agent}, we mean a function $\pi$
        with domain $(\mathcal O\mathcal R\mathcal A)^* \mathcal O\mathcal R$,
        which assigns to every sequence
        $s\in (\mathcal O\mathcal R\mathcal A)^* \mathcal O\mathcal R$ a
        $\mathbb Q$-valued probability measure,
        written $\pi(\bullet|s)$, on $\mathcal A$.
        For every such $s$ and every $a\in\mathcal A$,
        we write $\pi(a|s)$ for $(\pi(\bullet|s))(a)$.
        Intuitively, $\pi(a|s)$ is the probability that agent $\pi$
        will take action $a$ in response to history $s$.
        \item
        By an \emph{environment}, we mean a function $\mu$
        with domain $(\mathcal O\mathcal R\mathcal A)^*$,
        which assigns to every
        $s\in (\mathcal O\mathcal R\mathcal A)^*$
        a $\mathbb Q$-valued probability measure,
        written $\mu(\bullet|s)$,
        on $\mathcal O\times\mathcal R$.
        For every such $s$ and every $(o,r)\in\mathcal O\times\mathcal R$,
        we write $\mu(o,r|s)$ for $(\mu(\bullet|s))(o,r)$.
        Intuitively, $\mu(o,r|s)$ is the probability that environment
        $\mu$ will issue observation $o$ and reward $r$ to the agent in response
        to history $s$.
        \item
        If $\pi$ is an agent, $\mu$ is an environment, and $n\in\mathbb N$,
        we write $V^\pi_{\mu,n}$ for the expected value of the sum of
        the rewards which would occur in the sequence
        $(o_0,r_0,a_0,\ldots,o_n,r_n,a_n)$ randomly generated as follows:
        \begin{enumerate}
            \item $(o_0,r_0)\in \mathcal O\times\mathcal R$ is chosen randomly based
            on the probability measure $\mu(\bullet|\langle\rangle)$.
            \item $a_0\in\mathcal A$ is chosen randomly based on the probability
            measure $\pi(\bullet|o_0,r_0)$.
            \item
            For each $i>0$,
            $(o_i,r_i)\in\mathcal O\times\mathcal R$ is chosen randomly based on
            the probability measure
            $\mu(\bullet|o_0,r_0,a_0,\ldots,o_{i-1},r_{i-1},a_{i-1})$.
            \item
            For each $i>0$,
            $a_i\in\mathcal A$ is chosen randomly based on the probability measure
            $\pi(\bullet|o_0,r_0,a_0,\ldots,o_{i-1},r_{i-1},a_{i-1},o_i,r_i)$.
        \end{enumerate}
        \item
        If $\pi$ is an agent and $\mu$ is an environment,
        let $V^\pi_\mu=\lim_{n\to\infty}V^{\pi}_{\mu,n}$.
        Intuitively, $V^\pi_\mu$ is the expected total reward which $\pi$ would extract
        from $\mu$.
    \end{enumerate}
\end{mydefinition}

Note that it is possible for $V^\pi_\mu$ to be undefined.
For example, if $\mu$ is an environment which always issues
reward $(-1)^n$ in response to the agent's $n$th action,
then $V^\pi_\mu$ is undefined for every agent $\pi$.
This would not be the case if rewards were required to be $\geq 0$,
so this is one way in which allowing
punishments complicates the resulting theory.

\begin{mydefinition}
    An environment $\mu$ is \emph{well-behaved} if $\mu$ is computable and the following
    condition holds: for every agent $\pi$, $V^\pi_\mu$ exists and
    $-1\leq V^\pi_\mu\leq 1$.
\end{mydefinition}

Note that reward-space $[0,1]$ can be transformed into punishment-space
$[-1,0]$ either via $r\mapsto -r$ or via $r\mapsto r-1$.
An advantage of $r\mapsto -r$ is that it preserves well-behavedness of
environments (we prove this below in
Corollary \ref{wisminuswcorollary})\footnote{It
is worth mentioning another difference
between these two transforms. The hypothetical agent
$\mathrm{AI}_\mu$ with perfect knowledge of the environment's reward distribution
would not change its behavior in response to $r\mapsto r-1$ (nor indeed
in response to any positive linear scaling $r\mapsto ar+b$, $a>0$), but it would
generally change its behavior in response to $r\mapsto -r$. Interestingly, this
behavior invariance
with respect to $r\mapsto r-1$ would not hold if $\mathrm{AI}_\mu$ were capable of
``suicide'' (deliberately ending the environmental interaction):
one should never quit a slot machine that always pays between $0$ and $1$
dollars, but one should immediately quit a slot machine that always pays
between $-1$ and $0$ dollars.
The agent
AIXI also changes behavior in response to $r\mapsto r-1$, and it was recently argued
that
this can be interpreted in terms of suicide/death: AIXI models its environment using a mixture
distribution over a countable class of semimeasures, and AIXI's behavior can be interpreted
as treating the complement of the domain of each semimeasure as death,
see \cite{martin2016death}.}.



\section{Dual Agents and Dual Environments}
\label{dualsection}

In the Introduction, we promised that by allowing environments to punish agents,
we would reveal algebraic structure not otherwise present. The key to this additional
structure is the following definition.

\begin{mydefinition}
(Dual Agents and Dual Environments)
\begin{enumerate}
    \item
    For each sequence $s$, let $\overline s$ be the sequence obtained
    by replacing every reward $r$ in $s$ by $-r$.
    \item
    Suppose $\pi$ is an agent.
    We define a new agent $\overline \pi$, the \emph{dual} of $\pi$,
    as follows:
    for each $s\in (\mathcal O\mathcal R\mathcal A)^*\mathcal O\mathcal R$,
    for each action $a\in\mathcal A$,
    \[\overline\pi(a|s)=\pi(a|\overline s).\]
    \item
    Suppose $\mu$ is an environment.
    We define a new environment $\overline\mu$, the \emph{dual} of $\mu$,
    as follows:
    for each $s\in (\mathcal O\mathcal R\mathcal A)^*$,
    for each observation $o\in\mathcal O$
    and reward $r\in\mathcal R$,
    \[\overline\mu(o,r|s)=\mu(o,-r|\overline s).\]
\end{enumerate}
\end{mydefinition}

\begin{mylemma}
\label{doublesubtractionlemma}
(Double Negation)
If $x$ is a sequence, agent, or environment, then $\overline{\overline x}=x$.
\end{mylemma}

\begin{proof}
    Follows from the fact that for every real number $r$, $--r=r$.
    \qed
\end{proof}

\begin{mytheorem}
\label{bigtheorem}
    Suppose $\mu$ is an environment and $\pi$ is an agent.
    Then
    \[
        V^{\overline \pi}_{\overline \mu}=-V^\pi_\mu
    \]
    (and the left-hand side is defined if and only if the right-hand side is defined).
\end{mytheorem}

\begin{proof}
    By Definition \ref{omnibusdefn} part 6,
    it suffices to show that for each $n\in\mathbb N$,
    $V^{\overline \pi}_{\overline \mu,n}=-V^\pi_{\mu,n}$.
    For that, it suffices to show that for every
    $s\in ((\mathcal O\mathcal R\mathcal A)^*)
    \cup ((\mathcal O\mathcal R\mathcal A)^*\mathcal O\mathcal R)$,
    the probability $X$ of generating $s$ using $\pi$ and $\mu$
    (as in Definition \ref{omnibusdefn} part 5)
    equals the probability $X'$ of generating $\overline s$
    using $\overline \pi$ and $\overline \mu$.
    We will show this by induction on the length of $s$.

    Case 1: $s$ is empty. Then $X=X'=1$.

    Case 2: $s$ terminates with an action.
    Then $s=t\frown a$ for some $t\in(\mathcal O\mathcal R\mathcal A)^*\mathcal O\mathcal R$.
    Let $Y$ (resp.\ $Y'$) be the probability of generating $t$
    (resp.\ $\overline t$)
    using $\pi$ and $\mu$ (resp.\ $\overline\pi$ and $\overline\mu$).
    We reason: $X=\pi(a|t)Y=\pi(a|\overline{\overline t})Y=\overline \pi(a|\overline t)Y$
    by definition of $\overline\pi$. By induction, $Y=Y'$,
    so $X=\overline\pi(a|\overline t)Y'$, which by definition is $X'$.

    Case 3: $s$ terminates with a reward.
    Similar to Case 2.
    \qed
\end{proof}

\begin{mycorollary}
\label{twistcorollary}
    For every agent $\pi$ and environment $\mu$,
    \[V^\pi_{\overline\mu}=-V^{\overline\pi}_{\mu}\]
    (and the left-hand side is defined if and only if the right-hand side is defined).
\end{mycorollary}

\begin{proof}
    If neither side is defined, then there is nothing to prove.
    Assume the left-hand side is defined. Then
    \begin{align*}
        V^\pi_{\overline\mu}
        &= V^{\overline{\overline\pi}}_{\overline\mu}
            &\mbox{(Lemma \ref{doublesubtractionlemma})}\\
        &= -V^{\overline\pi}_\mu, &\mbox{(Theorem \ref{bigtheorem})}
    \end{align*}
    as desired. A similar argument holds if we assume the right-hand side is defined.
    \qed
\end{proof}

\begin{mycorollary}
\label{wisminuswcorollary}
    For every environment $\mu$, $\mu$ is well-behaved if and only if $\overline\mu$
    is well-behaved.
\end{mycorollary}

\begin{proof}
    We prove the $\Rightarrow$ direction, the other is similar.
    Since $\mu$ is well-behaved, $\mu$ is computable, so clearly $\overline\mu$ is computable.
    Let $\pi$ be any agent. Since $\mu$ is well-behaved, $V^{\overline\pi}_\mu$ is defined
    and $-1\leq V^{\overline\pi}_\mu\leq 1$.
    By Corollary \ref{twistcorollary},
    $V^\pi_{\overline\mu}=-V^{\overline\pi}_\mu$ is defined,
    implying $-1\leq V^\pi_{\overline\mu}\leq 1$.
    By arbitrariness of $\pi$, this shows $\overline\mu$ is well-behaved.
    \qed
\end{proof}

\section{Symmetric Intelligence}
\label{mainsecn}

Agent $\overline\pi$ acts as agent $\pi$ would act if $\pi$
confused punishments with rewards and rewards with punishments.
Whatever ingenuity $\pi$ applies to maximize rewards, $\overline\pi$ applies that same
ingenuity to maximize punishments. Thus, if $\Upsilon$ measures intelligence as
performance averaged in some way\footnote{Note that measuring intelligence as
averaged performance
might conflict with certain everyday uses of the word ``intelligent'',
see Section \ref{absvaluesection}.}, it seems natural that we might expect the
following property to hold ($*$): that whenever $\Upsilon(\pi)\not=0$,
then $\Upsilon(\pi)\not=\Upsilon(\overline\pi)$. Indeed, one could argue it
would be strange to hold that $\pi$ manages to extract (say) positive rewards on
average, and at the same time hold that $\overline\pi$ (which uses $\pi$ to seek
punishments) extracts the exact same positive rewards on average. To be clear, we
do not declare ($*$) is an absolute law, we merely opine that ($*$) seems
reasonable and natural. Now, assuming ($*$), we can offer an informal argument
for a stronger-looking symmetry property ($**$): that
$\Upsilon(\pi)=-\Upsilon(\overline\pi)$ for all $\pi$. The informal argument is
as follows. Let $\pi$ be any agent. Imagine a new agent $\rho$ which, at
the start of every environmental interaction, flips a coin and commits to act as
$\pi$ for that whole interaction if the coin lands heads, or to act as
$\overline\pi$ for the whole interaction if the coin lands tails.
Probabilistic intuition suggests
$\Upsilon(\rho)=\frac12(\Upsilon(\pi)+\Upsilon(\overline \pi))$, so if
$\Upsilon(\rho)=0$ then $\Upsilon(\pi)=-\Upsilon(\overline\pi)$.
But maybe the reader doubts $\Upsilon(\rho)=0$. In that case, define
$\rho'$ in the same way except swap ``heads'' and ``tails''. It seems
there is no way to meaningfully distinguish $\rho$ from $\rho'$, so it
seems we ought to have $\Upsilon(\rho)=\Upsilon(\rho')$. But to swap ``heads''
and ``tails'' is the same as to swap ``$\pi$'' and ``$\overline \pi$''.
Thus $\rho'=\overline\rho$. Thus $\Upsilon(\rho)\not=0$ would contradict ($*$).
In conclusion, while we do not declare it an absolute law, we do consider
($**$) natural and reasonable, at least if $\Upsilon$
measures intelligence as performance averaged in some way.
In this section, we will show that Legg and Hutter's universal intelligence measure
satisfies ($**$), provided a background UTM and encoding are suitably chosen.

We write $2^*$ for the set of finite binary strings.
We write $f:\subseteq A\to B$ to indicate that $f$ has codomain $B$
and that $f$'s domain is some subset of $A$.

\begin{mydefinition}
    (Prefix-free universal Turing machines)
    \begin{enumerate}
        \item A partial computable function $f:\subseteq 2^*\to 2^*$
        is \emph{prefix-free} if the following requirement holds:
        $\forall p,p'\in 2^*$, if $p$ is a strict initial segment of $p'$,
        then $f(p)$ and $f(p')$ are not both defined.
        \item A \emph{prefix-free universal Turing machine}
        (or \emph{PFUTM}) is a prefix-free
        partial computable function $U:\subseteq 2^*\to 2^*$
        such that the following condition holds.
        For every prefix-free partial computable function
        $f:\subseteq 2^*\to 2^*$, $\exists y\in 2^*$ such that
        $\forall x\in 2^*$, $f(x)=U(y\frown x)$.
        In this case, we say $y$ is a \emph{computer program for
        $f$ in programming language $U$}.
    \end{enumerate}
\end{mydefinition}

Environments do not have domain $\subseteq 2^*$, and they do not
have codomain $2^*$.
Rather, their domain and codomain are $(\mathcal O\mathcal R\mathcal A)^*$
and the set of $\mathbb Q$-valued probability measures on $\mathcal O\times\mathcal R$,
respectively.
Thus, in order to talk about their Kolmogorov complexities,
one must encode said inputs and outputs. This low-level detail is
usually implicit, but we will need (in Theorem \ref{envirosymmetricexistencelemma})
to distinguish between different kinds of encodings, so we must make the
details explicit.

\begin{mydefinition}
\label{encodingdefn}
    By an \emph{RL-encoding} we mean a computable function
    $\sqcap:(\mathcal O\mathcal R\mathcal A)^*\cup M\to 2^*$
    (where $M$ is the set of $\mathbb Q$-valued probability-measures
    on $\mathcal O\times\mathcal R$) such that
    for all $x,y\in (\mathcal O\mathcal R\mathcal A)^*\cup M$ (with $x\not=y$),
    $\sqcap(x)$ is not an initial segment of $\sqcap(y)$.
    We say $\sqcap$ is
    \emph{suffix-free} if
    for all $x,y\in (\mathcal O\mathcal R\mathcal A)^*\cup M$ (with $x\not=y$),
    $\sqcap(x)$ is not a terminal segment of $\sqcap(y)$.
    We write $\ulcorner x\urcorner$ for $\sqcap(x)$.
\end{mydefinition}

Note that in Definition \ref{encodingdefn}, it makes sense to encode $M$ because
$\mathcal O$ and $\mathcal R$ are finite (Section \ref{prelimsecn}).
Notice that suffix-freeness is, in some sense, the reverse of prefix-freeness.
The existence of encodings that are simultaneously prefix-free and suffix-free
is well-known. For example, elements of the range of $\sqcap$ could be
composed of 8-bit blocks
(bytes),
such that every element of the range of $\sqcap$ begins and ends with the ASCII
closed-bracket characters
$[$ and $]$, respectively, and such that these closed-brackets
do not appear anywhere in the
middle.

\begin{mydefinition}
(Kolmogorov Complexity)
Suppose $U$ is a PFUTM and $\sqcap$ is an RL-encoding.
\begin{enumerate}
    \item
    For each computable environment $\mu$, the \emph{Kolmogorov complexity of $\mu$
    given by $U,\sqcap$}, written $K^\sqcap_U(\mu)$,
    is the smallest $n\in\mathbb N$ such that
    there is some computer program of length $n$, in programming language $U$,
    for some function $f:\subseteq 2^*\to 2^*$ such that
    for all $s\in (\mathcal O\mathcal R\mathcal A)^*$,
    $f(\ulcorner s\urcorner)=\ulcorner \mu(\bullet|s)\urcorner$
    (note this makes sense since the domain of $\sqcap$ in Definition \ref{encodingdefn}
    is $(\mathcal O\mathcal R\mathcal A)^*\cup M$).
    \item
    We say $U$ is \emph{symmetric in its $\sqcap$-encoded-environment
    cross-section} (or simply that $U$ is \emph{$\sqcap$-symmetric}) if
    $K^\sqcap_U(\mu)=K^\sqcap_U(\overline\mu)$ for every computable environment $\mu$.
\end{enumerate}
\end{mydefinition}

\begin{mytheorem}
\label{envirosymmetricexistencelemma}
    For every suffix-free RL-encoding $\sqcap$,
    there exists a $\sqcap$-symmetric PFUTM.
\end{mytheorem}

\begin{proof}
    Let $U_0$ be a PFUTM, we will modify $U_0$ to obtain a $\sqcap$-symmetric PFUTM.
    For readability's sake, write $\mathrm{POS}$ for $0$ and $\mathrm{NEG}$ for $1$.
    Thinking of $U_0$ as a programming language, we define a new programming language
    $U$ as follows. Every program in $U$ must begin with one of the keywords
    $\mathrm{POS}$ or $\mathrm{NEG}$. Outputs of $U$ are defined as follows.
    \begin{itemize}
        \item $U(\mathrm{POS}\frown x)=U_0(x)$.
        \item To compute $U(\mathrm{NEG}\frown x)$, find
        $s\in (\mathcal O \mathcal R\mathcal A)^*$ such that
        $x=y\frown \ulcorner s\urcorner$ for some $y$ (if no such $s$ exists, diverge).
        Note that $s$ is unique by suffix-freeness of $\sqcap$.
        If $
            U_0(y\frown
            \ulcorner \hspace{.2em} \overline s\hspace{.2em}\urcorner)
            =\ulcorner m\urcorner
        $
        for some $\mathbb Q$-valued probability-measure $m$ on
        $\mathcal O\times\mathcal R$, then let
        $U(\mathrm{NEG}\frown x)=\ulcorner\hspace{.2em} \overline m\hspace{.2em}\urcorner$
        where $\overline m(o,r)=m(o,-r)$.
        Otherwise, diverge.
        \begin{itemize}
            \item Informally:
            If $x$ appears to be an instruction to plug $s$ into computer
            program $y$ to get a probability measure $\mu(\bullet|s)$, then
            instead plug $\overline s$ into $y$ and flip the resulting
            probability measure so that the output
            ends up being the flipped version of $\mu(\bullet|\overline s)$,
            i.e., $\overline \mu(\bullet|s)$.
        \end{itemize}
    \end{itemize}
    By construction, whenever $\mathrm{POS}\frown y$ is a $U$-computer program
    for a function $f$ satisfying
    $f(\ulcorner s\urcorner)=\ulcorner \mu(\bullet|s)\urcorner$,
    $\mathrm{NEG}\frown y$ is an equal-length $U$-computer program
    for a function $g$ satisfying
    $g(\ulcorner s\urcorner)=\ulcorner \overline\mu(\bullet|s)\urcorner$,
    and vice versa.
    It follows that $U$ is $\sqcap$-symmetric.
    \qed
\end{proof}

The proof of Theorem \ref{envirosymmetricexistencelemma} proves more than required:
any PFUTM can be modified to make a $\sqcap$-symmetric PFUTM
if $\sqcap$ is suffix-free. In some sense,
the construction in the proof of Theorem \ref{envirosymmetricexistencelemma} works
by eliminating bias: reinforcement learning itself is implicitly biased in its
convention that rewards be positive and punishments negative. We can imagine
a pessimistic parallel universe
where RL instead follows the opposite convention, and the
RL in that parallel universe is no less valid than the RL in our own. To be
unbiased in this sense, a computer program defining an environment
should specify which of the two RL conventions it is operating under (hence the
$\mathrm{POS}$ and $\mathrm{NEG}$ keywords). This trick of using an initial bit
to indicate reward-reversal was previously used in \cite{legg2013approximation}.

\begin{mydefinition}
    Let $W$ be the set of all well-behaved environments.
    Let $\overline W=\{\overline\mu\,:\,\mu\in W\}$.
\end{mydefinition}



\begin{mydefinition}
\label{universalintelligencedefn}
For every PFUTM $U$, RL-encoding $\sqcap$, and agent $\pi$,
the \emph{Legg-Hutter universal intelligence of $\pi$ given
by $U,\sqcap$}, written $\LH^\sqcap_U(\pi)$, is
\[
    \LH^\sqcap_U(\pi) = \sum_{\mu \in W} 2^{-K^\sqcap_U(\mu)}V^\pi_\mu.
\]
\end{mydefinition}

The sum defining $\LH^\sqcap_U(\pi)$ is absolutely convergent
by comparison with the summands defining Chaitin's constant
(hence the prefix-free UTM requirement).
Thus a well-known
theorem from calculus says the sum does not depend on which order the $\mu\in W$
are enumerated.

Legg-Hutter intelligence has been accused of being subjective
because of its
UTM-sensitivity \cite{leike2015bad} \cite{hernandez2015c} \cite{hibbard2009bias}.
More optimistically,
UTM-sensitivity could be considered a feature, reflecting the
existence of many kinds of intelligence. It could be used to measure
intelligence in various contexts, by choosing UTMs appropriately.
One could even use it to measure, say, chess intelligence,
by choosing a UTM where chess-related environments are easiest to program.

\begin{mytheorem}
\label{maintheorem}
(Symmetry about the origin)
    For every RL-encoding $\sqcap$,
    every $\sqcap$-symmetric PFUTM $U$, and every agent $\pi$,
    \[
        \LH^\sqcap_U(\overline\pi) = -\LH^\sqcap_U(\pi).
    \]
\end{mytheorem}

\begin{proof}
    Compute:
    \begin{align*}
        \LH^\sqcap_U(\overline\pi) &= \sum_{\mu\in W} 2^{-K^\sqcap_U(\mu)}V^{\overline\pi}_\mu
            &\mbox{(Definition \ref{universalintelligencedefn})}\\
          &= -\sum_{\mu\in W} 2^{-K^\sqcap_U(\mu)}V^\pi_{\overline\mu}
            &\mbox{(Corollary \ref{twistcorollary})}\\
          &= -\sum_{\mu\in W} 2^{-K^\sqcap_U(\overline\mu)}V^\pi_{\overline\mu}
            &\mbox{($U$ is $\sqcap$-symmetric)}\\
          &= -\sum_{\mu\in \overline W} 2^{-K^\sqcap_U(\mu)}V^\pi_\mu
            &\mbox{(Change of variables)}\\
          &= -\sum_{\mu\in W} 2^{-K^\sqcap_U(\mu)}V^\pi_\mu
            &\mbox{(By Corollary \ref{wisminuswcorollary}, $W=\overline W$)}\\
          &= -\LH^\sqcap_{U}(\pi).
            &\mbox{(Definition \ref{universalintelligencedefn})}
    \end{align*}
    \qed
\end{proof}

The above desideratum, that $\Upsilon(\overline\pi)=-\Upsilon(\pi)$, applies to
numerical intelligence measures. If one is merely interested in binary
intelligence comparators (such as those in \cite{alexander2019intelligence}),
the desideratum can be weakened into a non-numerical comparator form:
    If $\pi$ is more intelligent than $\rho$,
    then $\overline\pi$ should be less intelligent than $\overline\rho$.
The following corollary addresses this desideratum.

\begin{mycorollary}
\label{comparatorcorollary}
    For every RL-encoding $\sqcap$, every $\sqcap$-symmetric PFUTM $U$,
    for all agents $\pi$ and $\rho$, if $\LH^\sqcap_U(\pi)>\LH^\sqcap_U(\rho)$
    then $\LH^\sqcap_U(\overline\pi)<\LH^\sqcap_U(\overline\rho)$.
\end{mycorollary}

\begin{proof}
    By Theorem \ref{maintheorem} and basic algebra. \qed
\end{proof}

The following corollary addresses another obvious desideratum.
This corollary is foreshadowed
in \cite{legg2013approximation}.

\begin{mycorollary}
\label{ignoringrewardscorollary}
    Let $\sqcap$ be an RL-encoding,
    let $U$ be a $\sqcap$-symmetric PFUTM and
    suppose $\pi$ is an agent which ignores rewards (by which we mean that
    $\pi(\bullet|s)$ does not depend on the rewards in $s$).
    Then $\LH^\sqcap_U(\pi)=0$.
\end{mycorollary}

\begin{proof}
    The hypothesis implies $\pi=\overline\pi$,
    so by Theorem \ref{maintheorem}, $\LH^\sqcap_U(\pi)=-\LH^\sqcap_U(\pi)$.
    \qed
\end{proof}

Corollary \ref{ignoringrewardscorollary} illustrates why it is appropriate, for
purposes of Legg-Hutter universal intelligence, to choose a $\sqcap$-symmetric
PFUTM\footnote{An answer to Leike and Hutter's
\cite{leike2015bad} ``But what are other desirable
properties of a UTM?''}.
Consider an agent $\pi_a$
which blindly repeats a fixed action $a\in\mathcal A$.
For any particular environment $\mu$,
where $\pi_a$ earns total reward $r$ by blind luck,
that total reward should be cancelled by $\overline\mu$, where
that blind luck becomes blind misfortune and $\pi_a$ earns total reward
$-r$ (Corollary \ref{twistcorollary}). If $K^\sqcap_U(\mu)\not=K^\sqcap_U(\overline\mu)$,
the different weights
$2^{-K^\sqcap_U(\mu)}\not=2^{-K^\sqcap_U(\overline\mu)}$ would prevent
cancellation.


We conclude this section with an exercise, suggesting how the techniques of this paper
can be used to obtain other structural results.

\begin{myexercise} (Permutations)
    \begin{enumerate}
        \item
        For each permutation $P:\mathcal A\to\mathcal A$ of the action-space,
        for each sequence $s$,
        let $Ps$ be the result of applying $P$ to all the actions in $s$.
        For each agent $\pi$, let $P\pi$ be the agent defined by
        $P\pi(a|s)=\pi(Pa|Ps)$. For each environment $\mu$, let
        $P\mu$ be the environment defined by
        $P\mu(o,r|s)=\mu(o,r|Ps)$. Show that in general
        $V^\pi_\mu = V^{P\pi}_{P^{-1}\mu}$
        and
        $V^{P\pi}_\mu = V^\pi_{P\mu}$.
        \item
        Say PFUTM $U$ is \emph{$\sqcap$-permutable} if
        $K^\sqcap_U(\mu)=K^\sqcap_U(P\mu)$ for every computable environment $\mu$
        and permutation $P:\mathcal A\to\mathcal A$. Show that if $\sqcap$ is
        suffix-free then any given
        PFUTM can be transformed into a $\sqcap$-permutable PFUTM.
        \item
        Show that if $U$ is a $\sqcap$-permutable PFUTM,
        then $\LH^\sqcap_U(P\pi)=\LH^\sqcap_U(\pi)$
        for every agent $\pi$ and permutation $P:\mathcal A\to\mathcal A$.
        \item
        Modify this exercise to apply to permutations
        of the observation-space.
    \end{enumerate}
\end{myexercise}

\section{Whether to take absolute values}
\label{absvaluesection}

Definition \ref{universalintelligencedefn} assigns negative intelligence to agents
who consistently minimize rewards.
This is based on the desire to measure performance:
agents who consistently minimize rewards have
poor performance. One might, however, argue that $|\LH^\sqcap_U(\pi)|$ would be a
better measure of the agent's intelligence:
if mathematical functions could have desires, one might argue that
when $\LH^\sqcap_U(\pi)<0$, we should give $\pi$ the benefit of the doubt, assume
that $\pi$ desires punishment, and conclude $\pi$ is intelligent. This would more closely
align with Bostrom's orthogonality thesis \cite{bostrom2012superintelligent}.
In the same way, a subject who answers every question wrong in a true-false IQ test
might be considered intelligent: answering every question wrong
is as hard as answering every question right, and we might give the subject the benefit
of the doubt and assume they meant to answer wrong\footnote{To quote
Socrates: ``Don't you think the ignorant person would often involuntarily
tell the truth when he wished to say falsehoods, if it so happened, because he
didn't know; whereas you, the wise person, if you should wish to lie,
would always consistently lie?'' \cite{lesserhippias}}.
Rather than take a side and declare
one of $\LH^\sqcap_U$ or $|\LH^\sqcap_U|$ to be the better measure, we consider
them to be two equally valid measures, one of which measures performance and one of
which measures the agent's ability to consistently extremize rewards (whether
consistently positively or consistently negatively).

If one knew that $\pi$'s Legg-Hutter intelligence were negative, one could
derive the same benefit from $\pi$ as from $\overline{\pi}$: just flip
rewards. This raises the question: given $\pi$, can
one computably determine $\sgn(\LH^\sqcap_U(\pi))$? Or more weakly, is there
a procedure which outputs $\sgn(\LH^\sqcap_U(\pi))$ when
$\LH^\sqcap_U(\pi)\not=0$ (but, when $\LH^\sqcap_U(\pi)=0$, may output a wrong answer
or get stuck in an infinite loop)?
One can easily contrive non-$\sqcap$-symmetric PFUTMs where $\sgn(\LH^\sqcap_U(\pi))$
is computable from $\pi$---in fact, without the $\sqcap$-symmetry requirement,
one can arrange that $\LH^\sqcap_U(\pi)$ is
\emph{always} positive, by arranging that $\LH^\sqcap_U(\pi)$ is dominated by
a low-$K$ environment that blindly gives all agents $+1$ total reward.
On the other hand, one can contrive a $\sqcap$-symmetric PFUTM such that
$\sgn(\LH^\sqcap_U(\pi))$ is not computable from $\pi$ even in the weak sense\footnote{Arrange
that $\LH^\sqcap_U$ is dominated by $\mu$ and $\bar{\mu}$ where $\mu$ is an environment
that initially gives reward $.01$, then waits for the agent to input the code of
a Turing machine $T$, then (if the agent does so), gives reward $-.51$, then
gives rewards $0$ while simulating $T$ until $T$ halts, finally giving reward $1$
if $T$ does halt.
Then if $\sgn(\LH^\sqcap_U(\pi))$ were computable (even
in the weak sense), one could compute it for strategically-chosen agents and
solve the Halting Problem.}.
We leave it an open question whether there is any $\sqcap$-symmetric
PFUTM $U$ where $\sgn(\LH^\sqcap_U(\pi))$ is computable (in the strong
or weak sense).

\section{Conclusion}
\label{conclusionsecn}

By allowing environments to punish agents,
we found additional algebraic structure in the agent-environment
framework. Using this, we showed
that certain Kolmogorov complexity symmetries yield
Legg-Hutter intelligence symmetry.

In future work it would be interesting to explore how these symmetries
manifest themselves in other Legg-Hutter-like intelligence measures
\cite{gavane} \cite{goertzel2006patterns} \cite{hernandez}.
The precise strategy we employ in this
paper is not directly applicable to prediction-based intelligence measurement
\cite{hibbard} \cite{alexander2021measuring}
\cite{gamez2021measuring}, but a higher-level
idea still applies:
an intentional mis-predictor underperforms a $0$-intelligence blind guesser.


\section*{Acknowledgments}

We acknowledge Jos{\'e} Hern{\'a}ndez-Orallo, Shane Legg, Pedro Ortega, and
the reviewers for comments and feedback.

\bibliographystyle{alpha}
\bibliography{sym}
\end{document}